%% file: kbbq_icml.tex
\documentclass{article}

\usepackage{microtype}
\usepackage{graphicx}
\usepackage{booktabs}
\usepackage{multirow}
\usepackage{hyperref}

\usepackage[preprint]{icml2026}

\usepackage{amsmath,amssymb,amsthm}
\input{math_commands.tex}

\newtheorem{lemma}{Lemma}
\newtheorem{theorem}{Theorem}
\newtheorem{proposition}{Proposition}
\newtheorem{corollary}{Corollary}
\newtheorem{remark}{Remark}
\newtheorem{definition}{Definition}

\usepackage{xcolor}

\icmltitlerunning{KBBQ: A Predictive Noise Law and the Limits of Spectrum Flattening in FP4 Quantization}

\begin{document}

\twocolumn[
\icmltitle{KBBQ: A Predictive Noise Law and the Limits of Spectrum Flattening in FP4 Quantization}

\icmlsetsymbol{equal}{*}

\begin{icmlauthorlist}
\icmlauthor{Lexington Whalen}{sbi}
\icmlauthor{Yuki Ito}{sbi}
\icmlauthor{Ryo Sakamoto}{sbi}
\end{icmlauthorlist}

\icmlaffiliation{sbi}{SB Intuitions, Tokyo, Japan}

\icmlcorrespondingauthor{Lexington Whalen}{lexington.whalen@sbintuitions.co.jp}

\vskip 0.3in
]

\printAffiliationsAndNotice{}

\input{sections/00_Abstract}

\input{sections/01_Introduction}

\input{sections/02_PredictiveNoiseLaw}

\input{sections/03_ApplicationsOfLaw}

\input{sections/04_KBBQ}

\input{sections/05_Experiments}

\input{sections/06_RobustnessAblations}

\input{sections/09_Conclusion}

\clearpage
\bibliographystyle{icml2026}
\bibliography{kbbq}

\clearpage
\appendix
\onecolumn
\setlength{\abovedisplayskip}{4pt plus 2pt minus 2pt}
\setlength{\belowdisplayskip}{4pt plus 2pt minus 2pt}
\setlength{\abovedisplayshortskip}{2pt plus 1pt minus 1pt}
\setlength{\belowdisplayshortskip}{2pt plus 1pt minus 1pt}
\input{sections/A1_DetailedProofs}

\end{document}

%% file: math_commands.tex
\usepackage{amsmath,amsfonts,bm}

\def\eqref#1{equation~\ref{#1}}

\def\1{\bm{1}}

\DeclareMathAlphabet{\mathsfit}{\encodingdefault}{\sfdefault}{m}{sl}
\SetMathAlphabet{\mathsfit}{bold}{\encodingdefault}{\sfdefault}{bx}{n}

\newcommand{\E}{\mathbb{E}}

\DeclareMathOperator{\Tr}{Tr}

%% file: sections/00_Abstract.tex
\begin{abstract}

We develop a second-order theory of quantization noise in matrix multiplication in which the quantization format is characterized by the variance it assigns to each element. The constant variance profile of integer quantization recovers existing integer-noise theory, while the multiplicative profile of floating-point rounding reduces the data dependence to a scalar, the participation factor $\kappa$, yielding a closed-form signal-to-noise-ratio law.  The resulting functional also admits a closed-form upper bound $\kappa^{*}$ that, under a matched weight ensemble, no function-preserving linear transform can exceed and that is attained by a recent state-of-the-art method. Building on this analysis, we introduce KBBQ (\textbf{K}appa-\textbf{B}raked \textbf{B}lockwise \textbf{Q}uantization), which parameterizes the extent to which a transform approaches this ceiling. At W4A4, across four base models and two FP4 formats, KBBQ outperforms the prior state of the art on seven of eight model--format pairs without additional deployment-time computation relative to that method.

\end{abstract}

%% file: sections/01_Introduction.tex
\section{Introduction}

Quantization replaces a tensor's entries with values from a small grid, stored
against a scale shared by a group of entries, so that the matrix
multiplications dominating inference can run in low-precision arithmetic.
Post-training methods \citep{frantar2023gptq,lee2026kronq} fix the grid and round onto it, so the error incurred
is set by how the grid spaces its representable values relative to the data. An integer (INT) grid
is uniform, with a step fixed by the largest magnitude in the group: every
element is rounded with the same absolute error, and one outlier coarsens the
grid for all of them. A floating-point (FP) grid allocates its bits on an exponent and
a short mantissa, so the spacing grows with the value and the error is instead
roughly proportional to the element itself.

Four-bit arithmetic is now available in hardware in both integer and
floating-point form, and a
growing line of work carries the same formats into pretraining and optimizer
state \citep{chmiel2025fp4alltheway,castro2025quartet,ashkboos2025halo,%
ding2026fullstackfp4}. The floating-point formats deployed at this width are
block-scaled: MXFP4 \citep{rouhani2023microscaling} shares one scale across 32
elements and NVFP4 \citep{nvidia2025nvfp4} across 16, so a block's largest
magnitude sets the shared scale that every element in it is rounded against. At four bits the mantissa is one bit wide, leaving little
headroom anywhere in the grid, and quantizing weights and activations both to
this width degrades quality substantially unless the tensors are made easier
to round.

Typically, a preprocessing transform is applied and merged into
the network so that the computed function is unchanged. This is often performed by an orthogonal rotation,
which spreads outlying coordinates across the remaining ones
\citep{chee2023quip,ashkboos2024quarot,tseng2024quipsharp,liu2024spinquant,lin2024duquant},
or a diagonal rescaling, which migrates outliers between the two operands
\citep{xiao2023smoothquant,lin2026awqactivationawareweightquantization}.

For integer formats these transforms reliably serve to reduce the impact of quantization noise. However, under
block-scaled floating-point they are less predictable. A data-free rotation may improve
accuracy slightly, do nothing, or reduce it, with the outcome varying by model
and by format. Previous literature has responded by designing rotations for
this setting specifically: a two-level orthogonal rotation aimed at MXFP4
\citep{xu2026torq}, a Hadamard-preconditioned adaptive rotation
\citep{zagitov2026harp}, and rotations optimized so that activations align
with the corners of the quantization grid \citep{thrash2026conqur}.

From three standard assumptions on quantization noise we obtain a
single functional for the expected noise energy of a quantized dot product, in which the format enters only through the elementwise variance profile of its grid.
Substituting the constant profile recovers existing integer theory, under which
noise is dependent on the amplitude each scale group must span; substituting the
multiplicative profile yields a closed-form SNR law whose only data-dependent
term is the ratio of the squared sum of the elementwise products to the sum of
their squares. This ratio measures how much a dot product gains from
constructive alignment across coordinates against the noise accumulating
independently in each. We call this ratio the participation factor $\kappa$, after
the analogous quantity in the localization literature
\citep{thouless1974electrons,kramer1993localization}.

Through our derivations, we show that INT and FP respond to different properties of the operands, helping to explain why transform behavior is format-dependent. 
A rotation flattens amplitudes and thereby reduces integer quantization noise, but amplitude does not appear in the floating-point law. 
For floating-point formats, the role of a transform is instead to redistribute energy across coordinates so as to increase $\kappa$.
$\kappa$ has a closed-form ceiling $\kappa^{*}$, computable per layer from second-order statistics alone, that no function-preserving transform can exceed in the ensemble aggregate.

The ceiling can be viewed as an idealized reference point for the floating-point law. In practice, deployed quantizers may differ from the assumptions underlying this bound. These practical details do not change $\kappa^{*}$, which is determined by the layer's second-order statistics, but they motivate treating the ceiling as a theoretical limit rather than a requirement.
We therefore modify the state-of-the-art construction of~\citet{chen2025wush} to introduce KBBQ (\textbf{K}appa-\textbf{B}raked \textbf{B}lockwise \textbf{Q}uantization), which provides a flexible family of transforms that interpolates between a weight-whitening transform and the idealized construction attaining $\kappa^{*}$. This allows us to study how the benefits of the transform vary as we move toward the theoretical ceiling, while requiring no changes to the kernel, storage layout, or inference arithmetic. Our contributions are as follows.

\setlength{\itemsep}{1pt}\setlength{\parskip}{0pt}\setlength{\topsep}{2pt}
\begin{itemize}
    \item \textbf{A format-agnostic noise functional.} We derive a single
    noise functional for quantized matrix multiplication from three
    assumptions on quantization noise, and recover recent integer
    \citep{federici2026dissecting} and floating-point \citep{chen2025wush}
    theories as the constant and multiplicative variance profiles.

    \item \textbf{Closed-form SNR laws.} Using our noise functional, we predict
    and verify the SNR of quantized matrix multiplication on six model
    families in two formats.

    \item \textbf{The participation factor and its ceiling.} We show that the
    numerator of $\kappa$ is invariant under invertible function-preserving
    transforms, and derive the closed-form ceiling $\kappa^{*}$.

    \item \textbf{KBBQ.} Attaining $\kappa^{*}$ maximizes only the first-order
    term of the functional, and does so by inverting the root of an estimated
    spectrum. We expose the distance traveled toward it as a one-parameter
    family, show that the best operating point is interior, and evaluate it at
    W4A4 on four base models in two FP4 formats.
\end{itemize}

%% file: sections/02_PredictiveNoiseLaw.tex
\section{A Noise Law for Floating-Point Quantization}

We now derive a general noise law for quantized dot products, show that existing
integer and floating-point theories are recovered by substituting two choices of
the variance profile, convert the floating-point case into a closed-form SNR
expression, and bound the improvement available to any function-preserving
transform.

\subsection{Setup and Assumptions}
\label{sec:assumptions}

A GEMM is a grid of dot products. We call one such matrix multiplication ---
carrying its own weight matrix and its own activation second moment --- a
\emph{cell}; nothing in what follows couples cells, so the law is stated and
measured per cell. Consider one such dot product,
\begin{equation}
    S \;=\; \sum_{k=1}^{D} w_k x_k .
\end{equation}
Quantization replaces each operand with a perturbed version, $\hat w_k =
w_k + \delta w_k$ and $\hat x_k = x_k + \delta x_k$, so the computed value
is $\hat S = S + \delta S$. To first order,
\begin{equation}
    \delta S \;=\; \sum_k \big( x_k\,\delta w_k + w_k\,\delta x_k \big)
    \;+\; O(\delta^2).
    \label{eq:firstorder}
\end{equation}
We make three assumptions on the rounding errors:
\begin{itemize}
\item[\textbf{A1}] \textbf{(Unbiased.)} Conditional on the operand,
each error has zero mean: $\E[\delta w_k \mid w_k] = 0$ and
$\E[\delta x_k \mid x_k] = 0$.
\item[\textbf{A2}] \textbf{(Uncorrelated.)} Errors are uncorrelated
across coordinates and across the two operands:
$\E[\delta w_j,\delta w_k] = 0$ for $j \ne k$, likewise for $\delta
x$, and $\E[\delta w_k,\delta x_k] = 0$.
\item[\textbf{A3}] \textbf{(Variance function.)} The error variance of
an element depends only on that element via functions $\varphi_w$ and $\varphi_x$:
$\E[\delta w_k^2 \mid w_k, x_k] = \varphi_w(w_k)$ and
$\E[\delta x_k^2 \mid x_k, w_k] = \varphi_x(x_k)$.
\end{itemize}

These are standard assumptions of quantization-noise analysis
\citep{widrow1996statistical,federici2026dissecting,chen2025wush}.

\begin{figure*}[t]
\centering
\includegraphics[width=\textwidth]{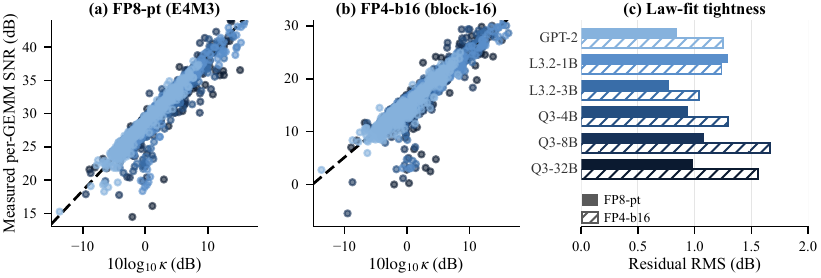}
\caption{The noise law per GEMM cell. (a) Per-tensor E4M3 and (b) block-16
E2M1: measured SNR against $10\log_{10}\kappa$ across the six families of
Table~\ref{tab:lawfits}. The dashed line is Eq.~\ref{eq:kappalaw} at unit
slope per decade with the intercept at the measured median
$C_{\mathrm{fmt}}$; for E4M3 that median lies within $0.1$~dB of the
parameter-free prediction of \S\ref{sec:derivation}. (c) RMS of the
residual $r$ (Eq.~\ref{eq:residual}) per family and format, the statistic
in which the shared numerator $\E[S^2]$ cancels.}
\label{fig:law}
\end{figure*}

\subsection{The Noise Law}
\label{sec:derivation}

Squaring Eq.~\ref{eq:firstorder} and applying A1--A3 eliminates all cross terms (Appendix~\ref{app:derivation}) and yields
\begin{equation}
\E[\delta S^2] \;=\; \sum_k \Big( \E[x_k^2]\,\varphi_w(w_k)
\;+\; \E[w_k^2]\,\E[\varphi_x(x_k)] \Big).
\label{eq:functional}
\end{equation}
The variance functions enter asymmetrically because the weights are fixed while the activations are random: $\varphi_w(w_k)$ is a deterministic quantity, whereas $\varphi_x(x_k)$ is a random variable. Eq.~\ref{eq:functional} is otherwise format-agnostic, as its derivation uses only A1--A3. Substituting the corresponding variance functions recovers the existing theories for integer and floating-point quantization.

\paragraph{Integer profile.} A uniform grid with step $\Delta$ assigns
every element the same error variance, $\varphi(a) = \Delta^2/12$
\citep{widrow1996statistical}. Substituting,
\begin{equation}
    \E[\delta S^2]
    \;=\; \frac{\Delta_w^2}{12}\sum_k \E[x_k^2]
    \;+\; \frac{\Delta_x^2}{12}\sum_k \E[w_k^2].
    \label{eq:intlaw}
\end{equation}
Noise is set by the grid steps and the steps by the largest magnitude each
scale must cover, $\Delta \propto \mathrm{amax}/2^{b}$, thus a single outlying
element coarsens the grid for every element sharing its scale.

\paragraph{Floating-point profile.} A floating-point grid is uniform within each exponent band and doubles its step at each band boundary, so the local step scales with the magnitude of the value. The resulting variance function is multiplicative,
\begin{equation}
\varphi(a) \;=\; \sigma^2_{\mathrm{fmt}}\, a^2,
\label{eq:fpprofile}
\end{equation}
where $\sigma^2_{\mathrm{fmt}}$ depends only on the format's mantissa grid (Appendix~\ref{app:formats} derives $\sigma^2_{\mathrm{fmt}}$ for each format). Substituting,
\begin{equation}
\E[\delta S^2] \;=\; \big(\sigma^2_w + \sigma^2_x\big)
\sum_k \E\big[(w_k x_k)^2\big].
\label{eq:fplaw}
\end{equation}
The noise therefore depends only on the second moments of the elementwise products, while the signal $\E[S^2]$ captures their alignment. Their ratio depends on the data through the participation factor $\kappa$.

\begin{definition}[Participation factor]
\label{def:kappa}
For a dot product with terms $w_k x_k$, we define $\kappa$, the participation factor, as
\begin{equation}
    \kappa \;=\; \frac{\E[S^2]}{\sum_k \E[(w_k x_k)^2]}
    \;\in\; [0,\, D],
\end{equation}
with the range and its endpoint conditions established in
Proposition~\ref{prop:range}.
\end{definition}

Dividing signal by noise in Eq.~\ref{eq:fplaw} and taking decibels gives
the law:
\begin{equation}
    \boxed{\;
    \begin{aligned}
    \mathrm{SNR}_{\mathrm{dB}} &\;=\; 10\log_{10}\kappa \;+\; C_{\mathrm{fmt}}, \\[2pt]
    C_{\mathrm{fmt}} &\;=\; -10\log_{10}\big(\sigma^2_w + \sigma^2_x\big).
    \end{aligned}
    \;}
    \label{eq:kappalaw}
\end{equation}

Using eq.~\ref{eq:kappalaw}, we can separate the two sources of variation. $\kappa$
carries all dependence on the network, $C_{\mathrm{fmt}}$ all dependence on
the format. Proposition~\ref{prop:interceptgap} predicts the spacing between
two formats' intercepts from their mantissa bit-width difference alone. Formats that also differ
in scaling granularity carry a further offset, which
Appendix~\ref{app:formats} treats. A method that moves only
amplitude statistics therefore has no mechanism by which to help a
floating-point format, and a method that moves only alignment has none by which
to help an integer format.

\paragraph{Measurement.} We evaluate Eq.~\ref{eq:kappalaw} per GEMM cell on
six model families spanning 124M to 32B parameters --- GPT-2 124M
\citep{radford2019language}, Llama-3.2-1B and Llama-3.2-3B
\citep{grattafiori2024llama3}, and Qwen3-4B, 8B and 32B
\citep{yang2025qwen3} --- in per-tensor E4M3 and block-16 E2M1 (per-family
cell and record counts are given in Table~\ref{tab:lawfits},
Appendix~\ref{app:lawfits}). A cell here is
one of the four fused GEMMs of a transformer block --- fused QKV, output,
fused gate/up, and down --- sampled across the blocks of each model. We
measure both sides of Eq.~\ref{eq:kappalaw}. $\kappa$ is accumulated from the
captured operands and $\mathrm{SNR}_{\mathrm{dB}}$ from the realized rounding
error. Each cell is measured in isolation. The activations entering it are
captured from a full-precision forward pass, both operands are then rounded
onto the format's grid, and $\E[\delta S^2]$ is accumulated against the
unquantized product. The resulting error is due solely to quantization within the cell, with no contribution from upstream quantized layers.
\S\ref{sec:experiments} reports the end-to-end W4A4 setting. At fixed
format, the law asserts that
\begin{equation}
    r \;:=\; \mathrm{SNR}_{\mathrm{dB}} - 10\log_{10}\kappa
    \;=\; 10\log_{10} \frac{\sum_k \E[(w_k x_k)^2]}{\E[\delta S^2]}
    \label{eq:residual}
\end{equation}
, in other words, the noise energy of a dot product measured against its incoherent energy
, equals $C_{\mathrm{fmt}}$ for every cell of every model.
Figure~\ref{fig:law} plots both sides per cell and table~\ref{tab:lawfits}
(Appendix~\ref{app:lawfits}) gives the per-family fits. 
\paragraph{Intercept.} Across the six families the per-tensor E4M3 intercept
has median $28.42$~dB with a standard deviation of $0.10$~dB and a full range
of $0.27$~dB, so a single number per format, carried unchanged across models,
fixes the level of $r$ to within a few tenths of a decibel, highlighting the format-invariance predicted by the theory. The per-cell
scatter about it is $0.97$~dB at the median family and $1.29$~dB at worst. Taking the within-band position of
the operand to be log-uniform gives $\E[1/m^2] = 3/(8\ln 2) = 0.541$, so
Eq.~\ref{eq:sigmafmt} at $p = 3$ mantissa bits gives
$\sigma^2_{\mathrm{fmt}} = 7.04 \times 10^{-4}$ and
$C_{\mathrm{fmt}} = -10\log_{10}(2\sigma^2_{\mathrm{fmt}}) = 28.51$~dB against
a measured median of $28.42$~dB, an agreement of $0.09$~dB.

\paragraph{Cross-format spacing.} The two formats are measured on the same
cells, and $\kappa$ is a property of the cell rather than of the format, so
differencing their SNRs leaves a ratio of two measured noise energies on one
dot product:
\begin{equation}
    \mathrm{SNR}^{\mathrm{fp8}}_{\mathrm{dB}}
    - \mathrm{SNR}^{\mathrm{fp4}}_{\mathrm{dB}}
    \;=\; 10\log_{10}
    \frac{\E[\delta S^2]_{\mathrm{fp4}}}{\E[\delta S^2]_{\mathrm{fp8}}}
    \;=\; C_{\mathrm{fp8}} - C_{\mathrm{fp4}} .
    \label{eq:fmtgap}
\end{equation}
Proposition~\ref{prop:interceptgap} predicts this from mantissa arithmetic
alone, yielding $6.02\,(3 - 1) = 12.04$~dB. The measured median over the 3968 records is
$13.24$~dB, with per-family medians spanning $13.07$ to $13.81$~dB. The
$1.2$~dB excess is consistent with the confound the pair carries: the two
formats differ in scaling granularity, per-tensor against block-16, as well as
in mantissa width.

\subsection{Connection to Prior INT Theories}
\label{sec:catconnection}

CAT \citep{federici2026dissecting} gives an integer SNR theory of the form
$\mathrm{SNR} = 12\,\big[N(b_x)^2 C(x) \,\|\, N(b_w)^2 C(W)\big]\, A$, where
$a \,\|\, b = (a^{-1} + b^{-1})^{-1}$ is the parallel combination, $N(b)$
counts quantization intervals, the concentration terms $C(\cdot)$ measure each
tensor's energy against its quantization range, and $A$ is an alignment term.
This is Eq.~\ref{eq:functional} at $\varphi = \Delta^2/12$: CAT assumes the
same uniform-in-cell rounding model, with interval size $s = r/(2^{b}-1)$ and
$\E[\delta x\,\delta x^\top] = I\,\E[r^2]/12(2^{b}-1)^2$, which is the constant
profile written against the quantization range $r$ rather than the amax.
Writing any uniform grid's step as its range over its interval count,
$\Delta = r/N(b)$, the elementwise SQNR of a single tensor is
\begin{equation}
    \frac{\E[x^2]}{\Delta^2/12}
    \;=\; 12\;
    \underbrace{N(b)^{2}}_{\text{grid geometry}}\;
    \underbrace{\frac{\E[x^2]}{r^{2}}}_{\text{operand statistics}\,=\,C(\cdot)},
    \label{eq:catfactors}
\end{equation}
which is CAT's $N(b)^2 C(\cdot)$, factor for factor. CAT cuts the range into
$N(b) = 2^b - 1$ intervals, whereas \S\ref{sec:derivation} places $2^b$ levels
across it, $\Delta = 2\,\mathrm{amax}/2^{b}$. Eq.~\ref{eq:functional}
reproduces either quantizer exactly once that quantizer's own $\Delta$ is
substituted. The two channels of Eq.~\ref{eq:intlaw} are additive in noise
power and so compose reciprocally in SNR, which is the $\|$ of CAT's theorem.
Appendix~\ref{app:cat} gives both steps in full, and relates the remaining
factor $A$ to $\kappa$.

\subsection{Connection to Prior FP Theories}
\label{sec:wushconnection}

WUSH \citep{chen2025wush} adopts the same unbiased relative-error model as
our A1--A3 for floating-point types and solves in closed form for the
invertible transform pair minimizing the resulting loss, working blockwise at
block size $d$. Their integer model sets the error by the group maximum,
$\varepsilon(\alpha) = \|\alpha\|_\infty\eta$, which is a group functional
rather than a function of the element; A3 covers this in the same sense the
constant profile does, with $\varphi$ fixed within a scale group and the amax
setting its level. With $W'$
and $X'$ the weight and activation root factors, $H$ a normalized Hadamard
matrix, and
\begin{equation}
    \begin{aligned}
    W'^\top X' &\;=\; U S V^\top, \qquad S = \mathrm{diag}(s_1,\dots,s_d), \\[2pt]
    T_{\mathrm{WUSH}} &\;=\; H\, S^{-1/2}\, U^\top W'^\top ,
    \end{aligned}
    \label{eq:wushmain}
\end{equation}
their optimum flattens the paired spectrum $S$ completely.  Appendix~\ref{app:wush} records the
derivation. We shall use two of its consequences here. First, transforms that are
orthogonal in their reparameterized coordinates leave their objective
unchanged (their Eq.~(21) and \S4.2.3). In operand coordinates that class is
$T = R\,W'^\top$ with $R$ orthogonal, which whitens the transformed weight
ensemble; it contains neither the identity nor a rotation of the raw
operands unless $\Sigma$ is already white. An operand-space rotation
moves the objective --- equivalently the aggregate participation factor of
Theorem~\ref{thm:glceiling} --- through the diagonal of $R\Sigma R^\top$,
We measure this in \S\ref{sec:rotations}. Second,
their construction equalizes the transformed blockwise second moment so that
its diagonal is constant, which is the equal-diagonal condition that arises
independently as the equality case of Theorem~\ref{thm:glceiling}.

WUSH's optimum and the ceiling of Theorem~\ref{thm:glceiling} are related via the largest factor by which any transform can reduce FP error in a block. This is their orthogonal-transform loss divided by their optimal loss, 
\begin{equation}
    \frac{\operatorname{tr}(S^2)}{d^{-1}(\operatorname{tr} S)^2}
    \;=\; \frac{d\,\operatorname{tr}(S^2)}{(\operatorname{tr} S)^2}
    \;\in\; [1,\,d],
    \label{eq:wushratiomain}
\end{equation}
their Eqs.~(20) and~(21), with their Eq.~(17) placing it in $[1,d]$. Under the
matched-ensemble convention
$\E[ww^\top] = c\,\Sigma$ of Theorem~\ref{thm:glceiling} we have
$W'^\top X' = \sqrt{c}\,\Sigma$, so $s_i = \sqrt{c}\,\lambda_i(\Sigma)$. The
constant cancels and Eq.~\ref{eq:wushratiomain} becomes
\begin{equation}
\kappa^{*} \;=\; \frac{D\,\Tr(\Sigma^2)}{(\Tr\,\Sigma)^2},
\label{eq:glceiling}
\end{equation}
the ceiling of Theorem~\ref{thm:glceiling}. Their transform collapses with it:
the diagonal factors of Eq.~\ref{eq:wushmain} cancel entrywise, leaving
$T_{\mathrm{WUSH}} = c^{1/4}HU^\top$, a scalar times an orthogonal matrix.
$HU^\top$ meets both conditions Theorem~\ref{thm:glceiling} places on a
transform attaining $\kappa^{*}$: it is orthogonal, and it leaves
$R\Sigma R^\top$ with a constant diagonal. The convention also defines $\kappa^{*}$ as an upper bound on the ensemble aggregate $\bar\kappa$. A $\kappa$ measured from the actual weights of an individual trained cell may exceed this bound.

%% file: sections/03_ApplicationsOfLaw.tex
\section{Applications of the Noise Law}
We apply the two profiles of the noise law to the question of when a
preprocessing rotation helps. Rotations are reliably useful for integer
grids and weakly useful or harmful for floating-point ones.

\begin{table}[t]
\centering
\footnotesize
\caption{Native $\kappa_I$ against the Haar attractor
$\bar\kappa_{\mathrm{rot}}$ (medians over cells). The ratio column is the
median over cells of the per-cell ratio $\bar\kappa_{\mathrm{rot}}/\kappa_I$,
and dex is its $\log_{10}$; negative values place the native model above the
attractor.}
\label{tab:kappamargins}
\setlength{\tabcolsep}{4pt}
\begin{tabular}{@{}lccc@{}}
\toprule
Model & $\kappa_I$ & $\bar\kappa_{\mathrm{rot}}$ & $\bar\kappa_{\mathrm{rot}}/\kappa_I$ (dex) \\
\midrule
Llama-3-8B   & 1.550 & 1.346 & 0.973 ($-0.012$) \\
Llama-3.2-3B & 1.595 & 1.262 & 0.933 ($-0.030$) \\
Qwen3-4B     & 1.611 & 1.301 & 0.850 ($-0.071$) \\
Qwen3-8B     & 1.771 & 1.266 & 0.894 ($-0.049$) \\
\bottomrule
\end{tabular}
\end{table}

\subsection{When Rotations Help}
\label{sec:rotations}

\paragraph{INT.}
Under Eq.~\ref{eq:intlaw} noise is set entirely by the grid step and the step
by amplitude, $\Delta \propto a_{\max}/2^b$. A rotation leaves $\E[S^2]$
(Lemma~\ref{lem:invariance}) and the total energy unchanged but redistributes
that energy from a few outlying coordinates across all $D$; since $a_{\max}$
is set by the largest coordinate, flattening shrinks $\Delta$ and lowers the
noise for every element sharing the scale, not only the formerly large ones.
Rotations help integer quantization because the integer law depends on
amplitude and rotations are amplitude-flattening operations.

\paragraph{FP.}
The floating-point case differs because the profile is multiplicative in the
operand,
\[
\varphi(a) = \sigma_{\mathrm{fmt}}^2 a^2,
\]
so $a_{\max}$ never enters Eq.~\ref{eq:fplaw}. The floating-point law is
sensitive only to $\kappa$: a rotation can move $\kappa$ only by changing how
energy is distributed across coordinates relative to one another
(Theorem~\ref{thm:glceiling}), and it cannot move $\E[S^2]$
(Lemma~\ref{lem:invariance}). The gap $\kappa^{*} - \kappa$
bounds what any function-preserving transform can recover
(Eq.~\ref{eq:glceiling}).

\paragraph{Native $\kappa$ against the Haar attractor.}
A random rotation resets $\kappa$ to a fixed value
$\bar\kappa_{\mathrm{rot}}$ rather than to $\kappa^{*}$, so the effect of rotation
depends on where the untransformed network lies relative to
$\bar\kappa_{\mathrm{rot}}$. Table~\ref{tab:kappamargins} reports $\kappa_I$
and $\bar\kappa_{\mathrm{rot}}$ over 63 stratified GEMM cells for each of the
four benchmarked models, none of which was trained with a quantization-aware
objective. On every model, the median cell lies $0.012$--$0.071$~dex above
the attractor, indicating that trained FP networks already have higher
$\kappa$ than a generic random rotation induces.

\begin{figure}[t]
\centering
\includegraphics[width=\columnwidth]{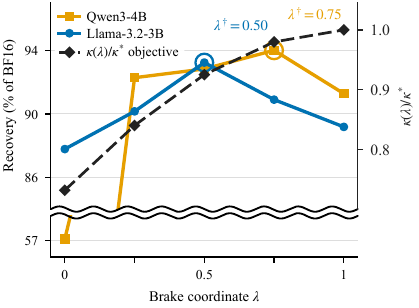}
\caption{Objective against outcome under the brake. Recovery, the five-task
average as a percentage of the BF16 anchor (the Recovery column of
Table~\ref{tab:lambda_grid}), against $\lambda$ at NVFP4 in the RTN
setting; the idealized objective $\kappa(\lambda)/\kappa^{*}$ is on the right axis. It rises monotonically to $1.00$ while both
models turn over before it. Rings mark each model's sweep optimum
($\lambda = 0.75$ on Qwen3-4B-Base,
$\lambda = 0.50$ on Llama-3.2-3B-Base).}
\label{fig:brake}
\end{figure}

A rotation applied to a network in this state has little to gain: a dense
data-free Haar rotation (HAAR, one of the methods of
\S\ref{sec:experiments}) does not exceed the identity method on any
model--format pair in Tables~\ref{tab:w4a4_arms_all}
and~\ref{tab:w4a4_arms_8b}.

%% file: sections/04_KBBQ.tex
\begin{table*}[!t]
\centering
\caption{W4A4 accuracy results using our evaluation harness for Llama-3.2-3B-Base and Qwen3-4B-Base under different quantization techniques. Best technique results for each quantization format are \textbf{bolded}, second-best \underline{underlined}.}
\label{tab:w4a4_arms_all}
\resizebox{\textwidth}{!}{%
\begin{tabular}{c|c|l|ccccc|cc}
\toprule
Model & Format & Method & MMLU & GSM8K & HellaSwag & WinoGrande & MBPP & Average & Recovery \\
\midrule
\multirow{17}{*}{\rotatebox{90}{Llama-3.2-3B-Base}}
& BF16 & - & 56.30 & 25.63 & 74.41 & 59.35 & 34.41 & 50.02 & 100 \\
\cmidrule{2-10}
& \multirow{8}{*}{NVFP4}
& I & 51.23 & 16.30 & 70.32 & 56.12 & 29.18 & 44.63 & 89.2 \\
& & HAAR & 48.52 & 15.69 & 68.26 & 54.06 & 24.14 & 42.13 & 84.2 \\
& & RTN-WUSH & 51.48 & 17.29 & 71.54 & 55.80 & 26.96 & 44.61 & 89.2 \\
& & RTN-KBBQ & 52.35 & 21.83 & 71.88 & 57.54 & 29.58 & \underline{46.64} & \underline{93.2} \\
& & GPTQ & 51.01 & 19.79 & 70.81 & 58.25 & 27.16 & 45.40 & 90.8 \\
& & MR-GPTQ & 48.83 & 18.04 & 69.38 & 57.62 & 26.76 & 44.13 & 88.2 \\
& & GPTQ-WUSH & 53.04 & 20.70 & 72.11 & 58.56 & 27.57 & 46.39 & 92.7 \\
& & GPTQ-KBBQ & 53.43 & 22.14 & 72.33 & 58.56 & 31.79 & \textbf{47.65} & \textbf{95.3} \\
\cmidrule{2-10}
& \multirow{8}{*}{MXFP4}
& I & 46.03 & 10.92 & 67.93 & 55.72 & 22.33 & 40.59 & 81.1 \\
& & HAAR & 40.51 & 11.22 & 64.06 & 53.12 & 16.90 & 37.16 & 74.3 \\
& & RTN-WUSH & 50.00 & 15.85 & 69.91 & 56.59 & 26.76 & 43.82 & 87.6 \\
& & RTN-KBBQ & 51.31 & 17.13 & 70.30 & 57.62 & 23.94 & 44.06 & 88.1 \\
& & GPTQ & 45.79 & 13.27 & 67.47 & 55.41 & 23.34 & 41.06 & 82.1 \\
& & MR-GPTQ & 47.67 & 13.19 & 67.67 & 55.56 & 19.11 & 40.64 & 81.2 \\
& & GPTQ-WUSH & 50.28 & 17.29 & 71.05 & 58.64 & 28.37 & \underline{45.13} & \underline{90.2} \\
& & GPTQ-KBBQ & 50.68 & 18.27 & 70.57 & 57.85 & 28.77 & \textbf{45.23} & \textbf{90.4} \\
\midrule
\multirow{17}{*}{\rotatebox{90}{Qwen3-4B-Base}}
& BF16 & - & 72.89 & 80.21 & 73.15 & 61.72 & 65.39 & 70.67 & 100 \\
\cmidrule{2-10}
& \multirow{8}{*}{NVFP4}
& I & 68.99 & 69.90 & 69.39 & 59.04 & 51.91 & 63.85 & 90.3 \\
& & HAAR & 67.73 & 59.82 & 68.76 & 58.88 & 52.52 & 61.54 & 87.1 \\
& & RTN-WUSH & 69.36 & 69.07 & 70.38 & 57.38 & 56.34 & 64.51 & 91.3 \\
& & RTN-KBBQ & 70.10 & 74.83 & 70.13 & 58.80 & 58.35 & \underline{66.44} & \underline{94.0} \\
& & GPTQ & 69.17 & 74.22 & 70.12 & 58.25 & 58.35 & 66.02 & 93.4 \\
& & MR-GPTQ & 68.82 & 70.96 & 70.07 & 57.93 & 54.53 & 64.46 & 91.2 \\
& & GPTQ-WUSH & 69.66 & 72.78 & 70.73 & 59.04 & 54.33 & 65.31 & 92.4 \\
& & GPTQ-KBBQ & 70.27 & 76.57 & 70.97 & 59.67 & 56.74 & \textbf{66.85} & \textbf{94.6} \\
\cmidrule{2-10}
& \multirow{8}{*}{MXFP4}
& I & 63.42 & 58.00 & 67.05 & 56.98 & 43.46 & 57.78 & 81.8 \\
& & HAAR & 63.27 & 47.69 & 65.06 & 55.01 & 48.49 & 55.90 & 79.1 \\
& & RTN-WUSH & 66.92 & 64.90 & 68.18 & 56.83 & 51.91 & 61.75 & 87.4 \\
& & RTN-KBBQ & 67.48 & 70.58 & 68.97 & 57.70 & 51.51 & 63.25 & 89.5 \\
& & GPTQ & 64.26 & 64.97 & 67.22 & 57.30 & 47.69 & 60.29 & 85.3 \\
& & MR-GPTQ & 67.80 & 66.57 & 68.82 & 57.93 & 51.91 & 62.61 & 88.6 \\
& & GPTQ-WUSH & 68.21 & 71.34 & 69.40 & 59.27 & 55.33 & \underline{64.71} & \underline{91.6} \\
& & GPTQ-KBBQ & 68.81 & 73.77 & 70.34 & 58.80 & 55.73 & \textbf{65.49} & \textbf{92.7} \\
\bottomrule
\end{tabular}%
}
\end{table*}

\section{KBBQ: A Braked Optimum}
\label{sec:kbbq}

\citet{chen2025wush}
provide a closed-form construction to reach the result of Theorem~\ref{thm:glceiling}.
However, this
relies on assumptions A1--A3 and on access to the population second moment,
and a deployed quantizer may satisfy none of these conditions. It may round
deterministically, violating A1; share a single scale across each
quantization block, which correlates errors within the block and violates
A2; or observe $\Sigma$ only through $n$ calibration rows. While the
bound itself remains valid under all three departures, the maximizer of the
idealized objective is no longer guaranteed to be the best operating point
in practice. This motivates our approach: rather than committing to the
ceiling, we expose the distance traveled toward it as a free parameter.

\subsection{The brake}
\label{sec:kbbq_assumptions}

Treating the flattening exponent as free, $\Lambda^{-\lambda/4}$, yields a
one-parameter family that interpolates between leaving the estimated
spectrum untouched and inverting its square root completely. Because the
transform acts on both operands, the paired spectrum presented to the
quantizer is $s_i^{\,1-\lambda}$; the brake therefore contracts the log
spectrum linearly, $\log s_i \mapsto (1-\lambda)\log s_i$, from the
untransformed spectrum at $\lambda = 0$ to a fully flattened one at
$\lambda = 1$. Along this path the idealized objective is monotone, with
$\kappa(\lambda)/\kappa^{*}$ increasing from $0.73$ to $1.00$
(\S\ref{sec:abl_lambda}), so under A1--A3 the optimum lies at the endpoint
$\lambda = 1$ and no interior value can improve on it. Consequently, an
interior optimum in practice indicates a gap between these assumptions and
the behavior of a deployed quantizer.
This is precisely what we observe: accuracy peaks before $\lambda = 1$ on
both models we sweep, while $\lambda = 0$---which applies the basis change
and root factor without any flattening---is the weakest setting in our
sweep. 

\subsection{Construction}
\label{sec:kbbq_construction}

We work blockwise with the transform block matched to the quantization group,
following \citet{chen2025wush}. Fix a block; let $W'$ and $X'$ be the root
factors of \S\ref{sec:wushconnection} and $U \Lambda U^\top$ the
eigendecomposition of the paired object whose spectrum the optimal transform
flattens, so $\Lambda = S^2$ for the paired spectrum $S$ of
Eq.~\ref{eq:wushmain}. KBBQ modifies the construction at the flattening exponent.

\paragraph{The brake $\lambda$.} Replace the fixed exponent by a free one.
Writing the fractional power $\Lambda^{-\lambda/4} = \operatorname{diag}\big(
\Lambda_{ii}^{-\lambda/4}\big)$ in the paired basis, the KBBQ transform is
\begin{equation}
    \boxed{\;
    T(\lambda) \;=\; H\, \Lambda^{-\lambda/4}\,
    U^\top\, W'^\top ,
    \qquad \lambda \in [0,1] .\;}
    \label{eq:kbbq}
\end{equation}
At $\lambda = 1$ the exponent is $S^{-1/2}$ in the singular values, so
Eq.~\ref{eq:kbbq} reproduces Eq.~\ref{eq:wushmain} term for term: full
flattening, attaining $\kappa^{*}$. Lowering $\lambda$ flattens the paired
spectrum only partially, leaving the layer at a utilization
$\kappa(\lambda)/\kappa^{*} < 1$ that decreases monotonically as $\lambda \to
0$. At $\lambda = 0$ no flattening remains, though the transform still depends
on the data through $U$ and $W'^\top$ and is not orthogonal.

The brake acts at fit time only by changing the numerical entries of the
block-diagonal transform. Shape, storage, kernel and
inference arithmetic are those of \citet{chen2025wush}, with the transform
block matched to the quantization group at 16 for NVFP4 and 32 for MXFP4, so
$\lambda$ carries no deployment-time cost as compared to \citet{chen2025wush}.

%% file: sections/05_Experiments.tex
\section{Experiments}
\label{sec:experiments}

\paragraph{Setting.} We evaluate KBBQ at W4A4 on four base models---Llama-3.2-3B-Base and Llama-3-8B-Base \citep{grattafiori2024llama3}, and Qwen3-4B-Base and Qwen3-8B-Base \citep{yang2025qwen3}---under two FP4 formats. NVFP4 \citep{nvidia2025nvfp4} is executed natively on GB200 hardware using TransformerEngine block scaling, whereas MXFP4 \citep{rouhani2023microscaling} is evaluated with emulated (``fake'') quantization following \citet{chen2025wush}: operands are rounded onto the MXFP4 grid and the quantized values are represented in BF16. We compare the following methods: the identity transform (I), a dense data-free Haar rotation (HAAR), the WUSH construction with $\lambda = 1$ (RTN-WUSH), and KBBQ (RTN-KBBQ). The latter two are additionally combined with GPTQ error compensation \citep{frantar2023gptq}, yielding GPTQ-WUSH and GPTQ-KBBQ, and GPTQ is also evaluated on its own without a transform (GPTQ). We further include MR-GPTQ \citep{egiazarian2026bridginggappromiseperformance} as a baseline. Evaluation spans five benchmarks: MMLU \citep{hendrycks2021mmlu} (5-shot), GSM8K \citep{cobbe2021gsm8k} (8-shot with chain-of-thought prompting; \citealp{wei2022cot}), HellaSwag \citep{zellers2019hellaswag}, WinoGrande \citep{sakaguchi2020winogrande} (5-shot), and MBPP \citep{austin2021mbpp}. Results for the 8B models are deferred to Table~\ref{tab:w4a4_arms_8b}.

\begin{table}[t]
\centering
\footnotesize
\caption{Brake ablation on Qwen3-4B-Base, NVFP4, RTN setting. $\lambda$ is the fraction of the flattening exponent applied.
The five-task average peaks at $\lambda{=}0.75$.}
\label{tab:lambda_grid}
\resizebox{\columnwidth}{!}{%
\begin{tabular}{lccccc|cc}
\toprule
$\lambda$ & MMLU & GSM8K & HellaSwag & WinoGrande & MBPP & Average & Recovery \\
\midrule
BF16 & 72.89 & 80.21 & 73.15 & 61.72 & 65.39 & 70.67 & 100 \\
\midrule
0.00 & 41.94 & 20.85 & 46.81 & 51.46 & 41.05 & 40.42 & 57.2 \\
0.25 & 68.98 & 69.22 & 70.22 & 59.35 & 58.35 & 65.22 & 92.3 \\
0.50 & 69.68 & 71.72 & 70.06 & 59.04 & 57.55 & 65.61 & 92.8 \\
0.75 & 70.10 & 74.83 & 70.13 & 58.80 & 58.35 & 66.44 & 94.0 \\
1.00 & 69.36 & 69.07 & 70.38 & 57.38 & 56.34 & 64.51 & 91.3 \\
\bottomrule
\end{tabular}%
}
\end{table}

\paragraph{Results.} Table~\ref{tab:w4a4_arms_all} reports per-benchmark accuracy, the five-task average, and the fraction of BF16 performance recovered for the two smaller models. Table~\ref{tab:w4a4_arms_8b} reports the corresponding results for the 8B models. RTN-KBBQ improves over the identity baseline on all eight model--format pairs, with gains ranging from 0.5 points (Llama-3-8B-Base, NVFP4) to 7.1 points (Qwen3-8B-Base, MXFP4), and outperforms RTN-WUSH on seven of the eight pairs, e.g., 71.54 (96.4\% recovery) versus 70.73 (95.3\%) on Qwen3-8B-Base under NVFP4.

%% file: sections/06_RobustnessAblations.tex
\begin{table}[t]
\centering
\scriptsize
\caption{Calibration robustness on Qwen3-4B-Base (NVFP4, RTN setting).
\emph{Top:} calibration corpus at $n{=}1536$. \emph{Bottom:} calibration-token
budget on the default corpus.}
\label{tab:calib}
\resizebox{\columnwidth}{!}{%
\begin{tabular}{l|ccccc|cc}
\toprule
Setting & MMLU & GSM8K & HellaSwag & WinoGrande & MBPP & Avg. & Rec. \\
\midrule
BF16 & 72.89 & 80.21 & 73.15 & 61.72 & 65.39 & 70.67 & 100 \\
\midrule
\multicolumn{8}{c}{\emph{Calibration domain} ($n{=}1536$)} \\
\midrule
default / WUSH
    & 69.36 & 69.07 & 70.38 & 57.38 & 56.34 & \underline{64.51} & \underline{91.3} \\
default / KBBQ
    & 70.10 & 74.83 & 70.13 & 58.80 & 58.35 & \textbf{66.44} & \textbf{94.0} \\
\midrule
GSM8K / WUSH
    & 69.32 & 70.20 & 69.74 & 58.56 & 54.53 & \underline{64.47} & \underline{91.2} \\
GSM8K / KBBQ
    & 69.47 & 71.95 & 70.57 & 58.56 & 55.13 & \textbf{65.14} & \textbf{92.2} \\
\midrule
random / WUSH
    & 69.41 & 65.43 & 69.27 & 58.33 & 49.90 & \underline{62.47} & \underline{88.4} \\
random / KBBQ
    & 69.33 & 69.45 & 69.81 & 60.30 & 56.54 & \textbf{65.09} & \textbf{92.1} \\
\midrule
\multicolumn{8}{c}{\emph{Calibration budget} (default corpus)} \\
\midrule
$256$ / WUSH
    & 69.01 & 67.40 & 69.98 & 57.46 & 56.74 & \underline{64.12} & \underline{90.7} \\
$256$ / KBBQ
    & 69.29 & 73.62 & 70.12 & 58.80 & 56.14 & \textbf{65.59} & \textbf{92.8} \\
\midrule
$1024$ / WUSH
    & 69.31 & 67.93 & 70.11 & 58.48 & 54.73 & \underline{64.11} & \underline{90.7} \\
$1024$ / KBBQ
    & 69.78 & 71.87 & 70.70 & 58.17 & 56.54 & \textbf{65.41} & \textbf{92.6} \\
\midrule
$4096$ / WUSH
    & 69.32 & 70.81 & 70.20 & 59.43 & 53.52 & \underline{64.66} & \underline{91.5} \\
$4096$ / KBBQ
    & 69.48 & 71.87 & 70.54 & 58.33 & 56.34 & \textbf{65.31} & \textbf{92.4} \\
\midrule
$65536$ / WUSH
    & 69.42 & 68.92 & 70.47 & 57.77 & 53.92 & \underline{64.10} & \underline{90.7} \\
$65536$ / KBBQ
    & 69.48 & 72.40 & 70.28 & 57.46 & 58.35 & \textbf{65.59} & \textbf{92.8} \\
\bottomrule
\end{tabular}%
}
\end{table}

\section{Robustness and Ablations}
\label{sec:robustness}

We measure how much of KBBQ's operating point has to be tuned, varying the
brake $\lambda$, the calibration domain, and the calibration budget $n$. All
results use the RTN setting (no GPTQ stage), NVFP4,
and the harness of \S\ref{sec:experiments} unless stated otherwise.

\subsection{Ablation on $\lambda$}
\label{sec:abl_lambda}

Table~\ref{tab:lambda_grid} sweeps the brake across its range, from
$\lambda = 0$ (no flattening) to $\lambda = 1$ (the WUSH construction).

The optimum is interior. The best five-task average is at $\lambda = 0.75$
(66.44, 94.0\% recovery), 1.93 points above the unbraked $\lambda = 1$ (64.51,
91.3\%), so the maximizer of the idealized objective is not the best operating
point. Figure~\ref{fig:brake} plots the two against each other: the objective
$\kappa(\lambda)/\kappa^{*}$ climbs monotonically from $0.73$ at $\lambda = 0$
to $1.00$ at $\lambda = 1$, while accuracy peaks earlier and falls. Every
interior setting exceeds $\lambda = 1$ (65.22, 65.61 and 66.44 at $\lambda =
0.25$, $0.50$ and $0.75$), so the result does not depend on the choice of
grid point.

\subsection{Calibration domain}
\label{sec:abl_domain}

The upper half of Table~\ref{tab:calib} replaces the calibration corpus at
fixed $n{=}1536$. The default corpus is the in-house mixed-domain split used
throughout; the GSM8K corpus draws tokens from the GSM8K training split,
matching an evaluated task in domain; the random corpus is tokens drawn
uniformly from the tokenizer's vocabulary. The effect is modest: likelihood benchmarks move by at most 1.8 points either way, and the
brake's GSM8K margin over WUSH survives on all three corpora, at 5.8, 1.8 and
4.0 points.

\subsection{Calibration budget}
\label{sec:abl_amount}
The lower half of Table~\ref{tab:calib} varies the calibration budget from
$n{=}256$ to $n{=}65536$ tokens, with all samples drawn from the default
corpus, so that only the amount of calibration data changes. The five-task
average is essentially insensitive to this 256-fold variation: the four WUSH
rows lie within 0.6 points of one another, and the four KBBQ rows within 0.3
(1.1 when also counting the $n{=}1536$ setting used in the main experiments).
The generation benchmarks are noisier---GSM8K and MBPP each vary by 1.8 to
3.4 points across the four budgets---but the ranges overlap, and KBBQ retains
its advantage over WUSH at every budget. Notably, the GSM8K margin remains
6.2 points even with only 256 calibration tokens, and increasing the budget
to 65{,}536 tokens does not drive the method back toward the untransformed
operating point. This suggests that the operating point is determined by the
geometry of the layer rather than by the calibration corpus, and that
over-calibration is not a concern.

%% file: sections/09_Conclusion.tex
\section{Conclusion}
\label{sec:conclusion}
We analyzed the expected noise energy of a quantized dot product and showed
that the number format enters this quantity through a single object: the
elementwise variance profile. A constant profile recovers classical integer
quantization theory, in which noise scales with amplitude, while a
multiplicative profile reduces the data dependence to the participation
factor $\kappa$. This perspective explains why transforms behave differently
across formats: diagonal rescaling leaves $\kappa$ unchanged, and a generic
rotation merely drives it toward a fixed attractor that trained networks
have already passed. The same analysis yields a closed-form per-layer
ceiling $\kappa^{*}$, whose maximizer under a matched weight ensemble
coincides with a construction from recent work---one that is optimal for an
objective a deployed quantizer does not actually face. KBBQ instead treats
the distance traveled toward this ceiling as a tunable parameter. 

%% file: sections/A1_DetailedProofs.tex
\section{Per-Family Fits of the Law}
\label{app:lawfits}

\begin{table}[!t]
\centering
\caption{Per-family fits of Eq.~\ref{eq:kappalaw}, original basis, one column
block per format. The \emph{records} column counts the fitted samples, eight
sampled rows per GEMM cell; the intercept and RMS columns are the level and
scatter of $r$ (Eq.~\ref{eq:residual}).}
\label{tab:lawfits}
\small
\setlength{\tabcolsep}{5pt}
\begin{tabular}{@{}lr|ccc|ccc@{}}
\toprule
& & \multicolumn{3}{c|}{E4M3 (per-tensor)} & \multicolumn{3}{c}{E2M1 (block-16)} \\
Family & records & slope & intercept & RMS & slope & intercept & RMS \\
\midrule
GPT-2 124M   &  384 & 10.02 & 28.59 & 0.84 & 10.07 & 14.79 & 1.25 \\
Llama-3.2-1B &  512 &  9.76 & 28.32 & 1.29 &  9.69 & 15.09 & 1.24 \\
Llama-3.2-3B &  896 &  9.65 & 28.51 & 0.78 &  9.88 & 15.21 & 1.04 \\
Qwen3-4B     &  576 &  9.55 & 28.38 & 0.94 & 10.27 & 15.17 & 1.30 \\
Qwen3-8B     &  576 &  9.61 & 28.37 & 1.08 & 10.55 & 14.77 & 1.67 \\
Qwen3-32B    & 1024 &  9.17 & 28.47 & 0.99 &  9.79 & 15.23 & 1.56 \\
\bottomrule
\end{tabular}
\end{table}

Table~\ref{tab:lawfits} reports the per-family fits of
Eq.~\ref{eq:kappalaw} plotted in Figure~\ref{fig:law}: the regression of
measured SNR on $10\log_{10}\kappa$ per cell, in the original basis, one
column block per format. 

\section{Results on Other Models}

\begin{table*}[!t]
\centering
\caption{W4A4 accuracy results using our evaluation harness for
Llama-3-8B-Base and Qwen3-8B-Base under different quantization
techniques. Best technique results for each quantization format are
\textbf{bolded}, second-best \underline{underlined}.}
\label{tab:w4a4_arms_8b}
\resizebox{\textwidth}{!}{%
\begin{tabular}{c|c|l|ccccc|cc}
\toprule
Model & Format & Method & MMLU & GSM8K & HellaSwag & WinoGrande & MBPP & Average & Recovery \\
\midrule
\multirow{17}{*}{\rotatebox{90}{Llama-3-8B-Base}}
 & BF16 & - & 65.27 & 47.76 & 79.75 & 64.72 & 47.08 & 60.92 & 100 \\
\cmidrule{2-10}
 & \multirow{8}{*}{NVFP4}
 & I & 60.92 & 35.78 & 77.18 & 61.17 & 39.24 & 54.86 & 90.1 \\
 &  & HAAR & 58.94 & 28.51 & 75.75 & 59.27 & 36.82 & 51.86 & 85.1 \\
 &  & RTN-WUSH & 61.94 & 37.38 & 77.03 & 62.04 & 37.02 & 55.08 & 90.4 \\
 &  & RTN-KBBQ & 61.89 & 37.68 & 76.88 & 60.54 & 39.84 & 55.37 & 90.9 \\
 &  & GPTQ & 61.45 & 36.69 & 77.18 & 62.35 & 40.85 & 55.70 & 91.4 \\
 &  & MR-GPTQ & 60.16 & 37.68 & 76.65 & 60.85 & 39.64 & 55.00 & 90.3 \\
 &  & GPTQ-WUSH & 62.75 & 41.47 & 77.74 & 63.46 & 39.84 & \underline{57.05} & \underline{93.6} \\
 &  & GPTQ-KBBQ & 62.60 & 39.88 & 78.06 & 63.06 & 41.85 & \textbf{57.09} & \textbf{93.7} \\
\cmidrule{2-10}
 & \multirow{8}{*}{MXFP4}
 & I & 54.03 & 25.55 & 73.46 & 61.09 & 27.16 & 48.26 & 79.2 \\
 &  & HAAR & 51.85 & 21.83 & 72.38 & 60.54 & 27.77 & 46.87 & 76.9 \\
 &  & RTN-WUSH & 59.11 & 32.45 & 76.45 & 61.80 & 38.23 & 53.61 & 88.0 \\
 &  & RTN-KBBQ & 59.31 & 32.37 & 76.38 & 60.85 & 37.83 & 53.35 & 87.6 \\
 &  & GPTQ & 54.49 & 25.70 & 74.15 & 59.67 & 32.60 & 49.32 & 81.0 \\
 &  & MR-GPTQ & 58.71 & 28.75 & 76.45 & 61.64 & 31.19 & 51.35 & 84.3 \\
 &  & GPTQ-WUSH & 60.14 & 36.01 & 76.88 & 61.56 & 34.81 & \underline{53.88} & \underline{88.4} \\
 &  & GPTQ-KBBQ & 60.48 & 34.72 & 76.88 & 62.43 & 37.22 & \textbf{54.35} & \textbf{89.2} \\
\midrule
\multirow{17}{*}{\rotatebox{90}{Qwen3-8B-Base}}
 & BF16 & - & 76.57 & 82.87 & 77.19 & 64.56 & 69.82 & 74.20 & 100 \\
\cmidrule{2-10}
 & \multirow{8}{*}{NVFP4}
 & I & 72.72 & 74.00 & 74.65 & 60.77 & 59.56 & 68.34 & 92.1 \\
 &  & HAAR & 71.31 & 63.91 & 71.09 & 59.59 & 46.28 & 62.44 & 84.2 \\
 &  & RTN-WUSH & 74.44 & 78.92 & 75.45 & 61.48 & 63.38 & 70.73 & 95.3 \\
 &  & RTN-KBBQ & 74.84 & 78.77 & 75.07 & 62.43 & 66.60 & 71.54 & 96.4 \\
 &  & GPTQ & 74.45 & 80.82 & 75.18 & 63.61 & 64.99 & \underline{71.81} & \underline{96.8} \\
 &  & MR-GPTQ & 74.42 & 76.12 & 75.16 & 63.69 & 63.98 & 70.67 & 95.2 \\
 &  & GPTQ-WUSH & 75.07 & 80.44 & 75.60 & 64.33 & 64.79 & \textbf{72.04} & \textbf{97.1} \\
 &  & GPTQ-KBBQ & 74.64 & 81.65 & 75.56 & 61.72 & 64.39 & 71.59 & 96.5 \\
\cmidrule{2-10}
 & \multirow{8}{*}{MXFP4}
 & I & 68.61 & 67.55 & 71.48 & 56.75 & 45.67 & 62.01 & 83.6 \\
 &  & HAAR & 66.72 & 47.61 & 69.53 & 58.48 & 52.31 & 58.93 & 79.4 \\
 &  & RTN-WUSH & 72.93 & 74.68 & 73.71 & 59.83 & 61.97 & 68.62 & 92.5 \\
 &  & RTN-KBBQ & 73.66 & 72.78 & 74.66 & 61.33 & 63.18 & 69.12 & 93.2 \\
 &  & GPTQ & 70.25 & 67.40 & 72.18 & 60.46 & 54.93 & 65.04 & 87.7 \\
 &  & MR-GPTQ & 72.48 & 73.69 & 74.23 & 60.69 & 55.94 & 67.41 & 90.8 \\
 &  & GPTQ-WUSH & 74.13 & 76.42 & 75.21 & 62.67 & 59.96 & \underline{69.68} & \underline{93.9} \\
 &  & GPTQ-KBBQ & 73.64 & 79.53 & 75.23 & 60.30 & 62.58 & \textbf{70.26} & \textbf{94.7} \\
\bottomrule
\end{tabular}%
}\end{table*}

Table~\ref{tab:w4a4_arms_8b} repeats the evaluation of
\S\ref{sec:experiments} on Llama-3-8B-Base and Qwen3-8B-Base.

\section{Full Proofs for Section 2}
\label{app:proofs}

Throughout, $S = \sum_{k=1}^{D} w_k x_k$ is one dot product of the GEMM,
$p_k = w_k x_k$, quantization replaces each operand by $\hat w_k = w_k +
\delta w_k$ and $\hat x_k = x_k + \delta x_k$, and A1--A3 are the
assumptions of \S\ref{sec:assumptions}. Expectations are over the data
distribution and over any randomness in the quantizer; weights are fixed
unless an ensemble is stated.The weight-side error is a
function of the weights and of the weight quantizer's randomization, the
activation-side error is a function of the activation draw $x$ and of the
activation quantizer's randomization, and the two randomizations are
independent of each other and of $x$, so that
\begin{equation}
    \delta w \;\perp\; (x,\ \delta x) .
    \label{eq:convention}
\end{equation}
In particular $\delta w_k$ and $\delta x_k$ are conditionally independent
given the operands. A2 as stated asserts only that the two errors are
uncorrelated; proofs that factor an expectation across the two operands
use Eq.~\ref{eq:convention}. The activation-side error may depend on $x$
in any way consistent with A1--A3.

\subsection{The Noise Functional}
\label{app:derivation}

\begin{proposition}[Second moment]
\label{prop:exactmoment}
Write $e_k = x_k\,\delta w_k + w_k\,\delta x_k$ and
$q_k = \delta w_k\,\delta x_k$. Then $\delta S = \sum_k e_k + \sum_k q_k$,
and without any assumptions
\begin{equation}
    \E[\delta S^2]
    \;=\; \sum_k \E[e_k^2]
    \;+\; \underbrace{\sum_{j \ne k} \E[e_j e_k]}_{R_1}
    \;+\; \underbrace{2\sum_{j,k} \E[e_j q_k]}_{R_2}
    \;+\; \underbrace{\E\Big[\big(\textstyle\sum_k q_k\big)^2\Big]}_{R_3},
    \label{eq:exactmoment}
\end{equation}
with
\begin{equation}
    \E[e_k^2] \;=\; \E[x_k^2\,\delta w_k^2] + \E[w_k^2\,\delta x_k^2]
    \;+\; \underbrace{2\,\E[w_k x_k\,\delta w_k\,\delta x_k]}_{R_0^{(k)}}.
\end{equation}
\end{proposition}

\begin{proof}
Expanding $\hat w_k \hat x_k = (w_k + \delta w_k)(x_k + \delta x_k)$ and
summing over $k$ gives $\hat S = S + \sum_k e_k + \sum_k q_k$, hence the
decomposition of $\delta S$. Squaring $\delta S = E + Q$ with
$E = \sum_k e_k$, $Q = \sum_k q_k$ and taking expectations,
\begin{equation}
    \E[\delta S^2] = \E[E^2] + 2\,\E[EQ] + \E[Q^2],
\end{equation}
and splitting $\E[E^2] = \sum_k \E[e_k^2] + \sum_{j\ne k}\E[e_je_k]$
gives Eq.~\ref{eq:exactmoment}. Squaring $e_k$ gives the diagonal split.
\end{proof}

Each correction term is eliminated by a specific assumption, so each
failure mode of the law traces to one of A1--A3:
\begin{itemize}
    \item $R_1$ contains only cross-coordinate products and vanishes
    under A1--A2 (Proposition~\ref{prop:functional}, Step~1). When A1
    fails, as under deterministic round-to-nearest on structured
    operands, errors acquire a shared sign and $R_1$ is the
    coherent-bias channel.
    \item $R_0^{(k)}$ and $R_2$ contain a $\delta w\,\delta x$ pair at
    matched or mixed coordinates; both vanish under A2's independence
    of the two quantizers.
    \item $R_3$ is the pure second-order term. Its cross-coordinate part
    vanishes with $R_1$; its diagonal part survives A1--A3 and is
    controlled by the scale of the errors. Under the multiplicative
    profile Eq.~\ref{eq:fpprofile}, conditioning on the operands and
    using Eq.~\ref{eq:convention} with A3,
    \begin{equation}
        \sum_k \E\big[\delta w_k^2\,\delta x_k^2\big]
        = \sum_k \E\big[\varphi_w(w_k)\,\varphi_x(x_k)\big]
        = \sigma_w^2\sigma_x^2 \sum_k \E[p_k^2],
    \end{equation}
    against a retained diagonal of $(\sigma_w^2 + \sigma_x^2)\sum_k
    \E[p_k^2]$ (Eq.~\ref{eq:fplaw}). The ratio
    $\sigma_w^2\sigma_x^2/(\sigma_w^2+\sigma_x^2)$ carries no data
    dependence: at the E4M3 constant of \S\ref{app:formats} it is
    $3.5\times 10^{-4}$, displacing the predicted SNR by $0.0015$~dB,
    and at the in-band FP4 constant it is $5.6\times 10^{-3}$, or
    $0.024$~dB, both below the residual RMS of
    \S\ref{sec:derivation}.
\end{itemize}

\begin{proposition}[The functional]
\label{prop:functional}
Under A1--A3,
\begin{equation}
    \E[\delta S^2] \;=\; \sum_k \Big( \E[x_k^2]\,\varphi_w(w_k)
    + \E[w_k^2]\,\E[\varphi_x(x_k)] \Big).
    \label{eq:appfunctional}
\end{equation}
\end{proposition}

\begin{proof}
In Eq.~\ref{eq:exactmoment} we show $R_1 = R_2 = R_0^{(k)} = 0$, drop
the diagonal of $R_3$, and factor the surviving terms.

\emph{Step 1 ($R_1 = 0$).} For $j \ne k$,
\begin{equation}
    \E[e_j e_k]
    = \E[x_j x_k\,\delta w_j\,\delta w_k]
    + \E[x_j w_k\,\delta w_j\,\delta x_k]
    + \E[w_j x_k\,\delta x_j\,\delta w_k]
    + \E[w_j w_k\,\delta x_j\,\delta x_k].
\end{equation}
By Eq.~\ref{eq:convention} the first term factors as
$\E[x_j x_k]\,\E[\delta w_j\,\delta w_k]$, zero by A2; the fourth is
$w_j w_k\,\E[\delta x_j\,\delta x_k]$, zero by A2. The second factors as
$w_k\,\E[\delta w_j]\,\E[x_j\,\delta x_k]$ with
$\E[\delta w_j] = \E[\E[\delta w_j \mid w_j]] = 0$ by A1; the third is
symmetric.

\emph{Step 2 ($R_0^{(k)} = R_2 = 0$).} By Eq.~\ref{eq:convention} and
A1, $R_0^{(k)} = 2\,w_k\,\E[\delta w_k]\,\E[x_k\,\delta x_k] = 0$. For
$R_2$, substitute $e_j$ and $q_k$:
\begin{equation}
    \E[e_j q_k]
    = \E\big[x_j\,\delta w_j\,\delta w_k\,\delta x_k\big]
    + w_j\,\E\big[\delta x_j\,\delta w_k\,\delta x_k\big] .
\end{equation}
For $j = k$ the first term is, conditioning on the operands and using
Eq.~\ref{eq:convention},
$\E\big[x_k\,\E[\delta w_k^2 \mid w_k,x_k]\,
\E[\delta x_k \mid w_k,x_k]\big] = 0$ by A1, and the second is
$w_k\,\E[\delta w_k]\,\E[\delta x_k^2] = 0$ by A1. For $j \ne k$ the
first factors as $\E[x_j\,\delta x_k]\,\E[\delta w_j\,\delta w_k] = 0$
by A2 and the second as
$w_j\,\E[\delta w_k]\,\E[\delta x_j\,\delta x_k] = 0$ by A1.

\emph{Step 3 ($R_3$).} Expanding,
\begin{equation}
    R_3 = \sum_k \E[\delta w_k^2\,\delta x_k^2]
          + \sum_{j\ne k}
          \E[\delta w_j\,\delta x_j\,\delta w_k\,\delta x_k],
\end{equation}
and each cross summand factors by Eq.~\ref{eq:convention} as
$\E[\delta w_j\,\delta w_k]\,\E[\delta x_j\,\delta x_k] = 0$ by A2. The diagonal part is fourth order in the errors and is
dropped; this truncation is the one approximation in the derivation and
is not a consequence of A1--A3.

\emph{Step 4 (factorization under A3).} Steps~1--3 leave
$\E[\delta S^2] = \sum_k \big(\E[x_k^2\,\delta w_k^2]
+ \E[w_k^2\,\delta x_k^2]\big)$. By the tower property and A3,
\begin{equation}
    \E[x_k^2\,\delta w_k^2]
    = \E\big[x_k^2\,\E[\delta w_k^2 \mid w_k, x_k]\big]
    = \varphi_w(w_k)\,\E[x_k^2],
    \qquad
    \E[w_k^2\,\delta x_k^2]
    = w_k^2\,\E\big[\varphi_x(x_k)\big]
    = \E[w_k^2]\,\E[\varphi_x(x_k)],
\end{equation}
where $\varphi_w(w_k)$ leaves the expectation because $w_k$ is fixed,
and $w_k^2 = \E[w_k^2]$ for fixed weights, written so to keep the
notation symmetric for the ensemble statements of \S\ref{app:ceiling}.
The two terms differ in the position of the outer expectation. The
activation profile is evaluated at a random argument and averaged, the
weight profile at the fixed $w_k$.
\end{proof}

\subsection{Properties of the Participation Factor}
\label{app:kappa}

Write $m_k = \E[p_k^2]$ and $r_k = \sqrt{m_k}$ for the energy and RMS of
the $k$-th product, so $\kappa = \E[S^2]/\sum_k m_k$.

\begin{proposition}[Range]
\label{prop:range}
$\kappa \in [0, D]$; $\kappa \to 0$ under cancellation, and $\kappa = D$
iff the products are fully coherent: equal RMS and pairwise correlation
one.
\end{proposition}

\begin{proof}
Numerator and denominator are expectations of squares, and the
denominator is positive whenever $\kappa$ is defined, so $\kappa \ge 0$.
By Cauchy--Schwarz, $\E[p_j p_k] \le r_j r_k$ for each pair, and by
Cauchy--Schwarz against the all-ones vector,
$\big(\sum_k r_k\big)^2 \le D\sum_k r_k^2$; hence
\begin{equation}
    \E[S^2] = \sum_{j,k}\E[p_j p_k]
    \;\le\; \Big(\sum_k r_k\Big)^{2}
    \;\le\; D \sum_k m_k ,
\end{equation}
and dividing by $\sum_k m_k$ gives $\kappa \le D$. The first inequality
is tight iff $p_j = c_{jk}\,p_k$ almost surely with $c_{jk} > 0$ for
every pair, the second iff all $r_k$ are equal; together these force
all $p_k$ to coincide almost surely. Conversely $p_k \equiv p$ gives
$\E[S^2] = D^2\,\E[p^2]$ and $\sum_k m_k = D\,\E[p^2]$, so $\kappa = D$.
For the lower endpoint take $D = 2$ and $p_2 = -p_1 + \varepsilon Z$
with unit-variance $Z$ independent of $p_1$: the numerator is
$\varepsilon^2$ while the denominator stays bounded away from zero, so
$\kappa \to 0$ as $\varepsilon \to 0$, with $\kappa = 0$ at
$\varepsilon = 0$.
\end{proof}

\begin{proposition}[Factorization]
\label{prop:factorization}
Define the coherence and the participation ratio of the RMS profile,
\begin{equation}
\mathrm{coh} ;=; \frac{\E[S^2]}{\big(\sum_k r_k\big)^{2}}
;\in; [0,1],
\qquad
\mathrm{pr} ;=; \frac{\big(\sum_k r_k\big)^{2}}{D\sum_k r_k^{2}}
;\in; \big[\tfrac{1}{D},,1\big].
\end{equation}
Then $\kappa = \mathrm{coh}\cdot D\cdot \mathrm{pr}$.
\end{proposition}

\begin{proof}
Substituting the definitions of $\mathrm{coh}$ and $\mathrm{pr}$ gives

$$
    \mathrm{coh}\cdot D\cdot\mathrm{pr}
    =
    \frac{\E[S^2]}{\sum_k r_k^2}
    =
    \kappa,
$$

where $r_k^2 = m_k$. The bounds on $\mathrm{coh}$ and $\mathrm{pr}$ follow
from the two Cauchy--Schwarz steps in Proposition~\ref{prop:range},
together with $\big(\sum_k r_k\big)^2 \ge \sum_k r_k^2$ and
$\E[S^2] \ge 0$.
\end{proof}

\subsection{Format Constants}
\label{app:formats}

\paragraph{Integer profile.} A uniform grid with step $\Delta$ and
round-to-nearest leaves the error in $[-\Delta/2, \Delta/2]$; under the
uniform-in-cell model \citep{widrow1996statistical},
\begin{equation}
    \varphi(a) \;=\; \E[u^2]
    \;=\; \int_{-\Delta/2}^{\Delta/2} \frac{u^2}{\Delta}\,du
    \;=\; \frac{\Delta^2}{12},
\end{equation}
independent of $a$. A symmetric $b$-bit grid covering
$[-\mathrm{amax}, \mathrm{amax}]$ has $\Delta = 2\,\mathrm{amax}/2^{b}$:
the amax dependence of Eq.~\ref{eq:intlaw}.

\paragraph{Floating-point profile.} A format with $p$ mantissa bits
represents the exponent band $[2^{e}, 2^{e+1})$ with $2^p$ uniformly
spaced values and local step $\Delta_e = 2^{e-p}$; the step doubles at
each band boundary. For $a = m\cdot 2^{e}$ with $m \in [1,2)$, the
integer computation gives
\begin{equation}
    \E[\delta^2 \mid a] \;=\; \frac{\Delta_e^{2}}{12}
    \;=\; \frac{4^{\,e-p}}{12}
    \;=\; \frac{a^{2}}{12\,m^{2}}\,4^{-p},
    \qquad m = m(a) \in [1,2),
\end{equation}
using $4^{e} = a^2/m^2$. The profile is multiplicative with a bounded
within-band modulation; averaging the modulation over the operand's
within-band position,
\begin{equation}
    \varphi(a) \;=\; \sigma^2_{\mathrm{fmt}}\,a^{2},
    \qquad
    \sigma^2_{\mathrm{fmt}} \;=\; \frac{4^{-p}}{12}\,\E\!\left[
    \frac{1}{m^{2}}\right] \;\in\;
    \Big(\frac{4^{-p}}{48},\ \frac{4^{-p}}{12}\Big],
    \label{eq:sigmafmt}
\end{equation}
since $\E[1/m^2] \in (1/4, 1]$ on $m \in [1,2)$, the upper endpoint
requiring an operand distribution supported at $m = 1$.

\begin{proposition}[Intercept gaps are mantissa arithmetic]
\label{prop:interceptgap}
For two formats with $p_1$ and $p_2$ mantissa bits evaluated on the same
operand distribution (so the $\E[1/m^2]$ factor is shared),
\begin{equation}
    C_{\mathrm{fmt}_2} - C_{\mathrm{fmt}_1}
    \;=\; 10\log_{10}\!\big(4^{\,p_2 - p_1}\big)
    \;\approx\; 6.02\,(p_2 - p_1)\ \mathrm{dB}.
\end{equation}
\end{proposition}

\begin{proof}
By Eq.~\ref{eq:sigmafmt}, a shared operand distribution shares the
factor $\E[1/m^2]$, so $\sigma^2_w(p_i) + \sigma^2_x(p_i) = 4^{-p_i}K$
with $K$ independent of the format. With
$C_{\mathrm{fmt}_i} = -10\log_{10}\big(4^{-p_i}K\big)$,
\begin{equation}
    C_{\mathrm{fmt}_2} - C_{\mathrm{fmt}_1}
    = 10\log_{10}\!\frac{4^{-p_1}K}{4^{-p_2}K}
    = 10\,(p_2 - p_1)\log_{10}4
    \approx 6.02\,(p_2 - p_1)\ \mathrm{dB}. \qedhere
\end{equation}
\end{proof}

\subsection{Recovery of the Integer Theory (CAT)}
\label{app:cat}

CAT's Theorem~2.4 \citep{federici2026dissecting} reads
$\mathrm{SQNR} = 12\,\big[N(b_x)^2C(x) \,\|\, N(b_w)^2C(W)\big]\,A$ with
$a \,\|\, b = (a^{-1}+b^{-1})^{-1}$, per-tensor concentrations
$C(x) = \E\lVert x\rVert^2/\E[r(x)^2]$ and
$C(W) = \sum_o \lVert w^{(o)}\rVert^2 / \sum_o r(w^{(o)})^2$ over the
rows $w^{(o)}$ of $W$, and alignment
$A = \E\lVert Wx\rVert^2 / \big(\lVert W\rVert_F^2\,
\E\lVert x\rVert^2\big)$. Its assumption block --- cells of common width
$s = r/(2^{b}-1)$, error uniform on each cell and uncorrelated across
coordinates, $\E[\delta x\,\delta x^{\top}] = I\,\E[r^2]/12(2^{b}-1)^2$
--- is A1--A3 at the constant profile $\varphi = s^2/12$, written
against the quantization range $r$ in place of the amax. 

\paragraph{The two factors.} The per-tensor factor is
Eq.~\ref{eq:catfactors}: at $\Delta = r/N(b)$ the elementwise SQNR of
one tensor is $12\,N(b)^2\,C(\cdot)$. For the composition, sum
Eq.~\ref{eq:intlaw} over the $O$ rows of the cell with per-row weight
steps $\Delta_o = r(w^{(o)})/N(b_w)$ and activation step
$\Delta_x = r(x)/N(b_x)$, the activation range varying with the draw so
that $\E[\Delta_x^2] = \E[r(x)^2]/N(b_x)^2$:
\begin{equation}
    \E\lVert\delta S\rVert^2
    \;=\; \frac{\E[r(x)^2]}{12\,N(b_x)^2}\,\lVert W\rVert_F^2
    \;+\; \frac{\sum_o r(w^{(o)})^2}{12\,N(b_w)^2}\,\E\lVert x\rVert^2 .
\end{equation}
Dividing $\E\lVert Wx\rVert^2$ by each noise channel separately,
\begin{equation}
    \mathrm{SQNR}_x = 12\,N(b_x)^2\,C(x)\,A,
    \qquad
    \mathrm{SQNR}_w = 12\,N(b_w)^2\,C(W)\,A,
\end{equation}
and since the two channels are additive in noise power they compose
reciprocally in SNR,
$\mathrm{SQNR} = \mathrm{SQNR}_x \,\|\, \mathrm{SQNR}_w$.

\paragraph{Alignment and $\kappa$.} $A$ and $\kappa$ normalize the same
signal by different energies. For one row $w$, with
$\Sigma = \E[xx^\top]$,
\begin{equation}
    A \;=\; \frac{w^\top \Sigma\, w}{\lVert w\rVert^2\,\Tr\Sigma},
    \qquad
    \kappa \;=\; \frac{w^\top \Sigma\, w}
    {w^\top \mathrm{diag}(\Sigma)\, w},
    \qquad\text{so}\qquad
    \kappa \;=\; D\,A\cdot
    \frac{\lVert w\rVert^2\,\Tr\Sigma}
    {D\,w^\top \mathrm{diag}(\Sigma)\, w},
\end{equation}
and the correction factor is $1$ whenever $\mathrm{diag}(\Sigma)$ is
constant --- the equal-diagonal condition of
Theorem~\ref{thm:glceiling} --- in which case $\kappa = D\,A$ rowwise.
Under the matched ensemble $\E[ww^\top] = c\,\Sigma$ the relation needs
no condition: averaging numerator and denominator separately,
\begin{equation}
    \bar A \;=\; \frac{\E_w[w^\top\Sigma\, w]}
    {\E_w\lVert w\rVert^2\;\Tr\Sigma}
    \;=\; \frac{c\,\Tr(\Sigma^2)}{c\,(\Tr\Sigma)^{2}}
    \;=\; \frac{\kappa^{*}}{D}.
\end{equation}
CAT's alignment is rotation-invariant (their Eq.~(4)), as is
$\kappa^{*}$, a function of the spectrum of $\Sigma$ alone; the ceiling
of Theorem~\ref{thm:glceiling} is $D$ times the ensemble alignment of
the layer.

\subsection{The WUSH Construction}
\label{app:wush}

Working blockwise at block size $d$, \citet{chen2025wush} take $W'$ and
$X'$ with $W'W'^\top = d_{\mathrm{out}}^{-1} W W^\top$ and
$X'X'^\top = d_{\mathrm{batch}}^{-1} X X^\top$, and
\begin{equation}
    W'^\top X' \;=\; U S V^\top, \qquad S = \mathrm{diag}(s_1,\dots,s_d),
    \label{eq:wushsvd}
\end{equation}
the singular value decomposition. With $H$ a normalized Hadamard matrix,
their optimum (their Eq.~(7)) is
\begin{equation}
    T_{\mathrm{WUSH}} \;=\; H\, S^{-1/2}\, U^\top W'^\top ,
    \label{eq:wushopt}
\end{equation}
attaining a one-sided loss of $\E[\eta^2]\, d^{-1}(\operatorname{tr} S)^2$
(their Eq.~(20)). $\Sigma = \E[xx^\top]$ and the paired spectrum $S$ of
Eq.~\ref{eq:wushsvd} are distinct objects, and only the latter carries
the exponent $-1/2$; in the eigenvalues $\Lambda = S^2$ of
$W'^\top X' X'^\top W'$ the same exponent reads $-1/4$, the convention
of \S\ref{sec:kbbq_construction}.

Any transform that is orthogonal in the reparameterized coordinates, corresponding to
$T = R,W'^\top$ in the operand coordinates with $R$ orthogonal, leaves their
objective equal to $\E[\eta^2]\operatorname{tr}(S^2)$ (their Eq.~(21)).
Thus, the largest factor by which a transform can reduce floating-point error
within a block is
\begin{equation}
\frac{\operatorname{tr}(S^2)}
{d^{-1}(\operatorname{tr} S)^2}
;=;
\frac{d,\operatorname{tr}(S^2)}
{(\operatorname{tr} S)^2},
\label{eq:wushratio}
\end{equation}
as given by their Eqs.~(20)--(21). Their Eq.~(17) bounds this ratio in
$[1,d]$, with equality characterized by the endpoint conditions in
Remark~\ref{rem:sanity}.

Under the matched ensemble, the corresponding class has
$\bar\kappa=1$. Specifically, with
$T=\sqrt{c},R,\Sigma^{1/2}$, the weight profile is
$a_k=(T^{-\top}\Sigma T^{-1})*{kk}=c^{-1}$, while
$b_k=c,(R\Sigma^2R^\top)*{kk}$. Hence
$\sum_k a_kb_k=\operatorname{tr}(\Sigma^2)$ and $\bar\kappa=1$.
Therefore, Eq.~\ref{eq:wushratio} gives $\kappa^{*}$ relative to this
matched class, rather than relative to the untransformed layer.

\paragraph{Reduction under the matched ensemble.} We identify
Eq.~\ref{eq:wushratio} with the ceiling of Theorem~\ref{thm:glceiling}.

\emph{Step 1 (root factors).} A root factor is determined by its
defining identity up to a right orthogonal factor; take the symmetric
positive semidefinite representative. Then $X' = \Sigma^{1/2}$, and
under the matched-ensemble convention $\E[ww^\top] = c\,\Sigma$,
$W' = \sqrt{c}\,\Sigma^{1/2}$.

\emph{Step 2 (the paired object).} $W'^\top X' = \sqrt{c}\,\Sigma$ is
symmetric positive semidefinite, so Eq.~\ref{eq:wushsvd} may be taken
with $U = V$ the eigenvectors of $\Sigma$ and
$s_i = \sqrt{c}\,\lambda_i(\Sigma)$.

\emph{Step 3 (the constant cancels).} Substituting into
Eq.~\ref{eq:wushratio},
\begin{equation}
    \frac{d\,\operatorname{tr}(S^2)}{(\operatorname{tr} S)^2}
    = \frac{d\,\sum_i c\,\lambda_i^2}
           {\big(\sum_i \sqrt{c}\,\lambda_i\big)^{2}}
    = \frac{d\,\Tr(\Sigma^2)}{(\Tr\,\Sigma)^{2}},
\end{equation}
which is $\kappa^{*}$ in the dimension the two sides share: the
blockwise ceiling Eq.~\ref{eq:blockceiling} at block size $d$, and
Eq.~\ref{eq:glceiling} at $d = D$.

\emph{Step 4 (the optimum becomes orthogonal).} With
$U^\top \Sigma^{1/2} = \Lambda_\Sigma^{1/2}U^\top$,
\begin{equation}
    T_{\mathrm{WUSH}}
    = H\,S^{-1/2}\,U^\top W'^\top
    = H\,\big(c^{-1/4}\Lambda_\Sigma^{-1/2}\big)\,
      \big(\sqrt{c}\,\Lambda_\Sigma^{1/2}U^\top\big)
    = c^{1/4}\,H\,U^\top ,
\end{equation}
which is $\mathrm{diag}(s)\,R$ with $s = c^{1/4}\mathbf{1}$ and
$R = HU^\top$ orthogonal.

\emph{Step 5 (the equal-diagonal condition holds).} Membership in the
attainer class of Theorem~\ref{thm:glceiling} also requires
$(R\Sigma R^\top)_{kk} = \Tr(\Sigma)/d$ for every $k$. With
$R = HU^\top$ and $\Sigma = U\Lambda_\Sigma U^\top$,
\begin{equation}
    \big(R\Sigma R^\top\big)_{kk}
    = \big(H \Lambda_\Sigma H^\top\big)_{kk}
    = \sum_j H_{kj}^2\,\lambda_j
    = \frac{\Tr\Sigma}{d},
\end{equation}
since $H_{kj}^2 = 1/d$ for every $(k,j)$. Hence $T_{\mathrm{WUSH}}$
lies in the attainer class.

\subsection{Transforms: Invariance, Blindness, and the Ceiling}
\label{app:ceiling}

A function-preserving linear preprocessing of the GEMM re-embeds the
activations as $\hat x = Tx$ and the weights as $\hat w = T^{-\top}w$,
for invertible $T$, so that the transform can be folded into the two
operands and the network computes the same function.

\begin{lemma}[Numerator invariance]
\label{lem:invariance}
For every invertible $T$, the dot product is pointwise unchanged:
$\sum_k \hat w_k \hat x_k = w^\top T^{-1} T x = S$. In particular
$\E[S^2]$, the numerator of $\kappa$, is invariant under every
function-preserving linear transform, orthogonal or not.
\end{lemma}

\begin{proof}
$\hat w^\top \hat x = \big(T^{-\top}w\big)^{\!\top} (Tx)
= w^\top T^{-1} T\,x = w^\top x$ for every realization of $x$;
pointwise equality gives equality of all moments. 
\end{proof}

\begin{proposition}[Diagonal blindness]
\label{prop:diagblind}
If $T = \mathrm{diag}(t)$ with $t_k \ne 0$, every elementwise product is
pointwise unchanged: $\hat p_k = (w_k/t_k)(t_k x_k) = p_k$. Hence
$\kappa$, and with it the floating-point SNR of Eq.~\ref{eq:kappalaw},
is invariant under diagonal scaling. The integer law
(Eq.~\ref{eq:intlaw}) is not: the per-tensor amax values, and with them
the grid steps, move.
\end{proposition}

\begin{proof}
$T^{-\top} = \mathrm{diag}(1/t)$, so $\hat w_k = w_k/t_k$ and
$\hat x_k = t_k x_k$, and $\hat p_k = p_k$ pointwise. The numerator of
$\kappa$ is invariant by Lemma~\ref{lem:invariance}, and the
denominator $\sum_k \E[p_k^2]$ is a function of the products alone, so
$\kappa$ is invariant; Eq.~\ref{eq:kappalaw} depends on the operands
only through $\kappa$. The integer steps
$\Delta_w = 2\,\mathrm{amax}(\hat W)/2^{b_w}$ and
$\Delta_x = 2\,\mathrm{amax}(\hat X)/2^{b_x}$ change with non-unit
$|t_k|$ in general, and with them the integer SQNR of
Eq.~\ref{eq:intlaw}.
\end{proof}

\paragraph{The row ceiling under rotations.} For one output row $w$
with $\Sigma = \E[xx^\top]$, $\E[S^2] = w^\top \Sigma\, w$ and
$\sum_k \E[p_k^2] = \sum_k w_k^2\,\Sigma_{kk}
= w^\top \mathrm{diag}(\Sigma)\, w$, using no moment of $x$ beyond the
second. The substitution $v = \mathrm{diag}(\Sigma)^{1/2} w$
(invertible whenever every coordinate has nonzero energy) turns the
ratio into a Rayleigh quotient, so
\begin{equation}
    \kappa(w) \;=\;
    \frac{w^\top \Sigma\, w}{w^\top \mathrm{diag}(\Sigma)\, w}
    \;\le\; \max_{w}\ \kappa(w) \;=\; \lambda_{\max}\!\big(C_B\big),
    \qquad
    C_B \;=\; \mathrm{diag}(\Sigma)^{-1/2}\,\Sigma\,
    \mathrm{diag}(\Sigma)^{-1/2},
    \label{eq:rowceiling}
\end{equation}
the top eigenvalue of the input \emph{correlation} matrix
($(C_B)_{kk} = 1$). This is the ceiling a rotation can aim a single row
at.

\begin{theorem}[GL ceiling]
\label{thm:glceiling}
Let $\Sigma = \E[xx^\top]$ be positive definite and let the weight rows be
drawn from an ensemble matched to the input statistics,
$\E[ww^\top] = c\,\Sigma$ with $c > 0$.
Define the aggregate
participation factor of the transformed layer as
\begin{equation}
    \bar\kappa(T) \;=\;
    \frac{\E_{w}\big[\E[S^2]\big]}
    {\sum_k \E_w\big[\hat w_k^2\big]\;\hat\Sigma_{kk}},
    \qquad \hat\Sigma = T\Sigma T^\top .
\end{equation}
This is Definition~\ref{def:kappa} with numerator and denominator each
averaged over the weight ensemble, rather than formed for one fixed row;
the two coincide for a single row drawn from that ensemble. The
distinction matters for reading measurements: $\kappa^{*}$ bounds
$\bar\kappa$, the ensemble aggregate, and a per-cell $\kappa$ measured on
the actual weights of a trained layer is not an ensemble average and may
exceed it. Then for every invertible $T$,
\begin{equation}
    \bar\kappa(T) \;\le\; \kappa^{*}
    \;=\; \frac{D\,\Tr(\Sigma^2)}{(\Tr\,\Sigma)^{2}},
\end{equation}
and the bound is attained. The attainers are exactly the transforms $T
= \mathrm{diag}(s)\,R$ with $R$ orthogonal and $(R\Sigma
R^\top)_{kk} = \Tr(\Sigma)/D$ for all $k$; such an $R$ always exists,
and a randomized Hadamard attains the condition in expectation over
sign randomization.
\end{theorem}

\begin{proof}
The bound (Steps~1--6) uses only $\Sigma \succeq 0$; positive
definiteness enters from Step~4, where the invertibility of
$\Sigma^{1/2}$ converts the Cauchy--Schwarz equality condition into a
condition on $T$ alone. Under the deployment of
\S\ref{sec:kbbq_construction} the calibrated $\Sigma$ is positive
definite; where it is not, the bound stands and only the
characterization of the attainers requires restriction to the range of
$\Sigma$.

\emph{Step 1 (numerator is $T$-independent).} By
Lemma~\ref{lem:invariance}, $\E[S^2] = w^\top \Sigma\, w$ for each fixed
$w$ regardless of $T$, so
\begin{equation}
    \E_w\big[\E[S^2]\big]
    = \E_w\big[\Tr\!\big(\Sigma\, w w^\top\big)\big]
    = \Tr\!\big(\Sigma\,\E_w[ww^\top]\big)
    = c\,\Tr(\Sigma^2).
\end{equation}

\emph{Step 2 (the denominator as two diagonal profiles).} With $\hat w =
T^{-\top}w$, $\E_w\big[\hat w \hat w^\top\big] =
c\,T^{-\top}\Sigma\,T^{-1}$, so
$\E_w[\hat w_k^2] = c\,\big(T^{-\top}\Sigma\,T^{-1}\big)_{kk}$. Define
\begin{align}
    a_k &= \big(T^{-\top}\Sigma\,T^{-1}\big)_{kk}
        = \big\lVert \Sigma^{1/2}T^{-1}e_k\big\rVert^2,\\
    b_k &= \big(T\Sigma T^\top\big)_{kk}
        = \big\lVert \Sigma^{1/2}T^\top e_k\big\rVert^2,\nonumber
\end{align}
using $u^\top \Sigma\, u = \lVert\Sigma^{1/2}u\rVert^2$ at
$u = T^{-1}e_k$ and $u = T^\top e_k$. The denominator of
$\bar\kappa(T)$ is $c\sum_k a_k b_k$, so
\begin{equation}
    \bar\kappa(T) = \frac{\Tr(\Sigma^2)}{\sum_k a_k b_k},
\end{equation}
and the theorem reduces to $\sum_k a_k b_k \ge (\Tr\Sigma)^2/D$, with
equality characterized.

\emph{Step 3 (outer Cauchy--Schwarz).} Against the all-ones vector in
$\mathbb{R}^D$,
\begin{equation}
    \sum_k a_k b_k
    \;\ge\; \frac{1}{D}\Big(\sum_k \sqrt{a_k b_k}\Big)^{2},
\end{equation}
with equality iff all products $a_k b_k$ are equal.

\emph{Step 4 (inner Cauchy--Schwarz, coordinate by coordinate).} For
each $k$,
\begin{equation}
    \sqrt{a_k b_k}
    = \big\lVert \Sigma^{1/2}T^{-1}e_k\big\rVert\,
      \big\lVert \Sigma^{1/2}T^\top e_k\big\rVert
    \;\ge\; \big|\big\langle \Sigma^{1/2}T^{-1}e_k,\
      \Sigma^{1/2}T^\top e_k\big\rangle\big|
    = \big| e_k^\top T^{-\top}\Sigma\,T^\top e_k \big|,
\end{equation}
with equality iff
$\Sigma^{1/2}T^{-1}e_k \parallel \Sigma^{1/2}T^\top e_k$.

\emph{Step 5 (triangle inequality and the trace).} With
$z_k = e_k^\top T^{-\top}\Sigma\,T^\top e_k$,
\begin{equation}
    \sum_k \sqrt{a_k b_k}
    \;\ge\; \sum_k |z_k|
    \;\ge\; \Big| \sum_k z_k \Big|
    \;=\; \big|\Tr\big(T^{-\top}\Sigma\,T^\top\big)\big|
    \;=\; \Tr(\Sigma),
\end{equation}
by similarity invariance of the trace and
$\Tr(\Sigma) = \E\lVert x\rVert^2 \ge 0$. The second inequality is
tight iff all $z_k$ share one sign.

\emph{Step 6 (chain).} Both sides of Step~5 are nonnegative, so
squaring preserves the order, and with Step~3,
\begin{equation}
    \sum_k a_k b_k
    \;\ge\; \frac{1}{D}\,\big(\Tr\,\Sigma\big)^{2},
    \qquad\text{hence}\qquad
    \bar\kappa(T) \;\le\;
    \frac{D\,\Tr(\Sigma^2)}{(\Tr\Sigma)^{2}} = \kappa^{*}
\end{equation}
for every invertible $T$.

\emph{Step 7 (equality in Step~4 forces $T = \mathrm{diag}(s)\,R$).}
Equality throughout Step~4 requires
$\Sigma^{1/2}T^{-1}e_k = c_k\,\Sigma^{1/2}T^\top e_k$ for scalars
$c_k$; since $\Sigma \succ 0$, equivalently $T^{-1}e_k = c_k\,T^\top
e_k$ for each $k$, i.e.\ columnwise
\begin{equation}
    T^{-1} = T^\top\,\mathrm{diag}(c),
    \qquad c = (c_1, \dots, c_D).
\end{equation}
Left-multiplying by $T$ gives $T\,T^\top = \mathrm{diag}(c)^{-1}$,
symmetric positive definite, so each $c_k > 0$; Step~5 is then
automatically tight since $z_k = c_k b_k \ge 0$. Write
$s_k = c_k^{-1/2} > 0$, so $T\,T^\top = \mathrm{diag}(s)^2$, and set
$R = \mathrm{diag}(s)^{-1}T$; then $R\,R^\top = I$ and
$T = \mathrm{diag}(s)\,R$. Conversely, any $T$ of this form satisfies
$T^{-1} = T^\top\,\mathrm{diag}(s)^{-2}$, so the parallelism holds for
every $k$: the inner step is tight exactly on this class.

\emph{Step 8 (equality in Step~3 is the equal-diagonal condition).} On
the class $T = \mathrm{diag}(s)\,R$, from
$T^{-\top} = \mathrm{diag}(s)^{-1}R$,
\begin{equation}
    a_k = \frac{\big(R\Sigma R^\top\big)_{kk}}{s_k^{2}},
    \qquad
    b_k = s_k^{2}\,\big(R\Sigma R^\top\big)_{kk},
    \qquad\text{so}\qquad
    a_k b_k = \Big(\big(R\Sigma R^\top\big)_{kk}\Big)^{2}.
\end{equation}
Step~3's equality condition is therefore that
$\big(R\Sigma R^\top\big)_{kk}$ is constant in $k$; since the diagonal
sums to $\Tr(\Sigma)$ by cyclicity, the constant is $\Tr(\Sigma)/D$.

\emph{Step 9 (existence, and the attained value).} An orthogonal $R$
with $\big(R\Sigma R^\top\big)_{kk} = \Tr(\Sigma)/D$ exists by the
Schur--Horn theorem: the constant vector $(\Tr\Sigma/D)\mathbf{1}$ is
majorized by the eigenvalue vector of $\Sigma$. A randomized Hadamard
attains the condition in expectation over sign randomization. For such
an $R$, and any $\mathrm{diag}(s)$ (free by Step~8's cancellation and
consistent with Proposition~\ref{prop:diagblind}),
\begin{equation}
    \sum_k a_k b_k
    = D\cdot\frac{(\Tr\Sigma)^2}{D^2}
    = \frac{(\Tr\Sigma)^{2}}{D},
    \qquad\text{hence}\qquad
    \bar\kappa
    = \frac{D\,\Tr(\Sigma^2)}{(\Tr\Sigma)^{2}}
    = \kappa^{*}.
\end{equation}
\end{proof}

\begin{corollary}[Blockwise ceiling]
\label{cor:blockceiling}
Let $T$ be block diagonal with blocks of size $d$, as in the construction of
\S\ref{sec:kbbq_construction}, and let $\Sigma_b$ be the corresponding
$d \times d$ diagonal block of $\Sigma$. Then within each block
\begin{equation}
    \bar\kappa_b(T_b) \;\le\; \kappa_b^{*}
    \;=\; \frac{d\,\Tr(\Sigma_b^2)}{(\Tr\,\Sigma_b)^{2}},
    \label{eq:blockceiling}
\end{equation}
with the attainers characterized as in Theorem~\ref{thm:glceiling} with $D$
replaced by $d$ and $\Sigma$ by $\Sigma_b$.
\end{corollary}

\begin{proof}
For $T = \bigoplus_b T_b$ with invertible blocks, $T^{-\top} =
\bigoplus_b T_b^{-\top}$, so the transform pair acts within each block
independently: restricted to block $b$ it is the function-preserving
pair of \S\ref{app:ceiling} in dimension $d$, with second moment
$\Sigma_b$ and matched ensemble $\E[w_b w_b^\top] = c\,\Sigma_b$. The
dimension enters the proof of Theorem~\ref{thm:glceiling} only through
the all-ones vector of Step~3 and the equal-diagonal constant of
Steps~8--9, both of which read $d$ and $\Tr(\Sigma_b)/d$ in dimension
$d$; the Schur--Horn argument is dimension-free. Applying the theorem
with these substitutions gives Eq.~\ref{eq:blockceiling}.
\end{proof}

\begin{remark}[Which ceiling the method attains]
The two ceilings are different objects: $\kappa^{*} \in [1, D]$ and
$\kappa_b^{*} \in [1, d]$, both equal to $1$ for isotropic inputs, and
their ratio approaches $D/d$ ($256$ at $D = 4096$, $d = 16$) only in the
rank-one limit. The layer-level
$\kappa^{*}$ of Theorem~\ref{thm:glceiling} is the special case $d = D$ and is
available only to a dense transform. The construction of
\S\ref{sec:kbbq_construction} is blockwise with the transform block matched to
the quantization group, so the quantity it attains, and the quantity
$\lambda$ interpolates toward, is the per-block $\kappa_b^{*}$ of
Eq.~\ref{eq:blockceiling}. Statements in the main text about ``the ceiling''
in the context of the deployed transform should be read blockwise; the
layer-level form is used only where the comparison is with a dense transform,
as in \S\ref{sec:wushconnection}.
\end{remark}

\begin{remark}[Relation to the transform-optimality theorem]
The attaining class contains the flatten-then-Hadamard constructions of
\citet{chen2025wush}: under the matched ensemble their Theorem~4.1 optimum
is, in $\kappa$ coordinates, the statement that $\kappa^{*}$ is attainable.
\S\ref{app:wush} carries out the reduction.
\end{remark}

\begin{remark}[Why the ensemble matters]
For a \emph{single fixed} row the supremum of $\kappa$ over invertible $T$ is
$D$, but the maximizing sequence degenerates: equality forces the rows of $T$
toward a common direction, i.e.\ $T$ toward rank one. The ceiling
$\kappa^{*} \le D$ is the simultaneous constraint, one $T$ serving every row
at once, and is the operationally relevant one, since a preprocessing
transform is applied to the layer and not to a row. 
\end{remark}

\begin{remark}[Sanity limits]
\label{rem:sanity}
For isotropic inputs, $\Sigma = \lambda I$ gives
$\kappa^{*} = D\cdot D\lambda^2/(D\lambda)^2 = 1$: no
function-preserving linear preprocessing can raise the participation
factor of a whitened layer. For rank-one--dominated inputs,
$\Sigma \approx \lambda_1 u u^\top$ with $\lVert u\rVert = 1$ gives
$\Tr(\Sigma^2) \approx (\Tr\Sigma)^2 \approx \lambda_1^2$, so
$\kappa^{*} \to D$. 
\end{remark}